\documentclass[conference]{IEEEtran}

\usepackage[
    top=0.73in,
    bottom=1.06in,
    left=0.625in,
    right=0.625in
]{geometry}

\usepackage[utf8]{inputenc}
\usepackage[T1]{fontenc}
\usepackage{cite}
\usepackage{graphicx}
\usepackage{psfrag}
\usepackage{subfigure}
\usepackage[caption=false]{subfig}
\usepackage{url}
\usepackage{ifthen}
\usepackage[cmex10]{amsmath}
\usepackage{array}
\usepackage{amssymb}
\usepackage{amsthm}
\usepackage{mathtools}
\usepackage{amsfonts}
\usepackage{epstopdf}
\usepackage{algorithm}
\usepackage{algorithmic}

\newtheorem{theorem}{Theorem}

\usepackage{setspace}
\usepackage{amscd}
\usepackage{mathrsfs}
\usepackage{epsfig}
\usepackage{color}
\usepackage{textcomp}
\usepackage{comment}
\usepackage{xcolor}
\usepackage{pgf}
\usepackage{pgfplots}
\pgfplotsset{compat=1.9}
\usepackage{tikz}
\usetikzlibrary{decorations.pathreplacing}
\usetikzlibrary{decorations.markings}
\usetikzlibrary{patterns}
\usepackage[outline]{contour}
\contourlength{1.2pt}
\pgfplotsset{compat=1.18}
\usetikzlibrary{plotmarks}
\usetikzlibrary{spy}
\usetikzlibrary{calc}
\usetikzlibrary{shapes,plotmarks,positioning,fit,chains,tikzmark,arrows,hobby,arrows.meta}
\usetikzlibrary{decorations.pathreplacing,calligraphy}
\usepgfplotslibrary{groupplots}
\usepackage{enumitem} 
\setlist{itemsep=0ex,parsep=0ex,topsep=1ex,partopsep=0ex,leftmargin=1em}

\usepackage[font=scriptsize]{caption}

\usepackage{import}

\IEEEoverridecommandlockouts

\DeclareMathOperator*{\argmax}{argmax}
\DeclareMathOperator*{\argmin}{argmin}

\usepackage{dsfont}
\usepackage{bbm}

\newcommand{\subtext}[1]{\text{\fontfamily{cmr}\fontshape{n}\fontseries{m}\selectfont{}#1}}
\newcommand{\sub}[1]{\ensuremath{_{\subtext{#1}}}} 

\def\BibTeX{{\rm B\kern-.05em{\sc i\kern-.025em b}\kern-.08em
    T\kern-.1667em\lower.7ex\hbox{E}\kern-.125emX}}

\begin{document}

\title{Conformal Sequential Decision Making With  Cost and Sublinear Policy Regret Guarantees
	\thanks{The work of O. Simeone was supported by an Open Fellowship of the EPSRC (EP/W024101/1), by the EPSRC (EP/X011852/1) and by the European Research Council (ERC) (No. 101198347).}
}

\author{%
	\IEEEauthorblockN{Jayadev Naram\IEEEauthorrefmark{1},
                      Giuseppe Durisi\IEEEauthorrefmark{1},
                      Matteo Zecchin\IEEEauthorrefmark{2},
                      and Osvaldo Simeone\IEEEauthorrefmark{3}}
	\IEEEauthorblockA{\IEEEauthorrefmark{1}
        Department of Electrical Engineering,
		Chalmers University of Technology,
		412 96 Göteborg, Sweden}
	\IEEEauthorblockA{\IEEEauthorrefmark{2}
        Communication Systems Department,
		EURECOM, 06410 Sophia Antipolis, France 
        }
	\IEEEauthorblockA{\IEEEauthorrefmark{3}
        Institute for Intelligent Networked Systems (INSI), 
        Northeastern University London,
        E1W 1LP London, UK}
		Email: \{jayadev, durisi\}@chalmers.se, zecchin@eurecom.fr, o.simeone@nulondon.ac.uk
}

\maketitle

\begin{abstract}
Automated decision makers, such as AI agents that control robots or network devices, must often satisfy hard cost constraints while maintaining competitive performance. In such settings, a cost measure, such as a safety violation rate, energy consumption, or a distortion level, must stay within a given budget in every episode of operation. Furthermore, a cumulative reward metric, accounting for per-step performance, must stay close to the reward of a meaningful reference policy. We formulate this setup as an episodic sequential decision problem in which the state sequence is arbitrary, and the reward and cost may depend on past actions within the same episode.  We propose the online conformal episodic agent (OCEA), a policy that selects one candidate action per episode using a regret-controlled algorithm, and switches to a zero-cost action when the remaining budget cannot absorb the maximum cost of one more round. We prove that OCEA satisfies per-episode cost constraints for every state sequence and for every realization of its internal randomness, while achieving a sublinear expected policy regret with respect to the best action sequence in hindsight from a comparator class satisfying the same per-episode cost constraints. We apply OCEA to risk-averse control and online compression with distortion guarantees.
\end{abstract}


\section{Introduction}

\subsection{Motivation}

In engineering fields such as robotics and telecommunications, automatic control systems that include AI agents are expected to become increasingly common~\cite{lindemann2025formal,simeone2026conformal}. These systems make decisions in real time, often with the help of learned predictors, and are expected to abide by the cost and performance requirements of the given application. 

Cost constraints must often hold strictly over each well-defined episode of operation, ensuring that the system cost does not exceed a given budget. Examples include delivery drones with fixed energy budgets between recharging events; robots with well-defined safety requirements within each subtask;  and communication devices with constraints on the number of outages within each transmission frame. In all these cases, a guarantee that holds only on average over a long time is not sufficient, because a single faulty episode can  cause a service outage.

In these scenarios, performance is often measured by a reward signal, such as quality of service or throughput, and the goal is to ensure that, under any conceivable sequence of environment states, the cumulative reward is not degraded, compared to a meaningful baseline, by more than an acceptable margin. In particular, since the actions affect the future rewards through the history of the system, the baseline must be ideally evaluated along its own trajectory~\cite{arora2012online}.

\subsection{Related Work}

Several existing frameworks, summarized next, address only parts of this problem, and this work aims to address gaps in the state of the art.

\noindent\emph{Online conformal prediction.} This framework produces prediction sets with long-term deterministic coverage guarantees~\cite{gibbs2021adaptive,angelopoulos2024online}, and  conformal risk control extends these guarantees to general bounded losses~\cite{angelopoulos2024conformal,zecchin2024localized}. These methods control a time-averaged cost over the full horizon, and they do not ensure that the cost is within budget in every episode.

\noindent\emph{Conformal methods for safe control and decision making.} Online conformal prediction \cite{gibbs2021adaptive} has been used for safe planning in dynamic environments~\cite{lindemann2023safe}, to calibrate model predictive control with probabilistic sequence models~\cite{lekeufack2024conformal,zecchin2024forking}, and to monitor temporal-logic requirements at runtime~\cite{cairoli2023conformal}. We refer to~\cite{lindemann2025formal} for a recent survey. Risk-averse decision making based on prediction sets was given decision-theoretic foundations in~\cite{kiyani2025decision}, with extensions presented in~\cite{zhu2026conformal,farzaneh2026}. In telecommunications, online conformal methods support compression with distortion guarantees~\cite{zecchin2026prediction} and the calibration of AI modules~\cite{simeone2026conformal}. None of these works simultaneously provides per-episode guarantees and regret guarantees on the reward.

\noindent\emph{Online optimization with constraints and policy regret.} Online convex optimization with long-term constraints allows sublinear cumulative constraint violation and standard regret~\cite{mahdavi2012trading,yu2017online}.  Unlike these frameworks, the cost budget in our problem is renewed in every episode. Furthermore, standard regret compares the learner's cumulative reward with the one obtained by playing a fixed action along the learner's own history. Policy regret, as defined in \cite{arora2012online}, accounts instead for the baseline's own trajectory. Sublinear policy regret is generally not possible against adversaries with unbounded memory; however, it can be obtained for bounded-memory adversaries by playing each action over blocks of rounds~\cite{arora2012online}. Nevertheless, these works do not include strict cost constraints.

\subsection{Contributions}

This paper studies a sequential decision problem that is divided into episodes. The state sequence is arbitrary and exogenous, and the reward and the cost may depend on the history within each episode. The main contributions are as follows.

\begin{itemize}
	\item We formulate a sequential decision problem with two requirements: a per-episode cost constraint that must hold for every state sequence and every realization of the learner's randomness, and sublinear expected policy regret with respect to the best feasible action sequence in a specified comparator class, with the learner and comparator each evaluated along the history induced by its own actions.
	\item We propose a policy, which we refer to as online conformal episodic agent (OCEA), that selects one candidate action per episode by using a regret-controlling algorithm~\cite{auer2002nonstochastic}, and that switches to a zero-cost action when the remaining budget cannot absorb the maximum cost of one more round.
	\item We prove that OCEA satisfies the per-episode cost constraint deterministically, and that its expected policy regret is sublinear in the time horizon $T$. 
	\item We apply the framework to risk-averse control~\cite{kiyani2025decision} and online  compression~\cite{ganesan2025online}. In both applications, actions correspond to thresholds that define prediction sets as in conformal prediction.
\end{itemize}

\section{Problem formulation}

\subsection{Setting}
\label{sec:general-setup}

We consider a sequential decision problem over a time horizon~$T$ that is partitioned into $J$ consecutive blocks, or episodes. As detailed in this section, the problem is characterized by an arbitrary exogenous state sequence, as well as by per-episode reward and cost metrics, on which we enforce deterministic constraints. We first describe the general setting, and then provide specific examples.

The $J$ episodes have respective lengths $n_1,\dots,n_J$, with $\sum_{j=1}^J n_j=T$. We index the $k$th round of block $j$ by $(j,k)$, where $j\in[J]$ and $k\in[n_j]$, with $[n]=\{1,\ldots,n\}$ for any integer $n$. The corresponding global time is
\begin{equation}
	\tau(j,k)=\sum_{r<j}n_r+k.
\end{equation} At round $k$ in block $j$, the decision maker selects, possibly following a randomized policy, an action $A_{j,k}\in\mathcal A=\{a_1,\dots,a_K\}$ where $2\le K<\infty$. After the action is selected, the state $X_{j,k}\in\mathcal X$ is revealed. The state sequence may be arbitrary, but is assumed to be exogenous, meaning that it is not affected by the learner's actions or randomization. The interaction history available before round $(j,k)$ is
\begin{equation}
	\label{eq:within-block-history}
	H_{j,k} = \bigl(A_{j,1:k-1},X_{j,1:k-1}\bigr),
\end{equation}
with $H_{j,1}=\emptyset$. 
Similarly, we write $H_{\tau(j,k)}=(A_{1:\tau(j,k)-1},X_{1:\tau(j,k)-1})$ for the history before global time $\tau(j,k)$. Note that the history is reset at the beginning of every block, enforcing a separation between different episodes.

The quality of action $A_{j,k}$ is evaluated via the reward
\begin{equation}
	r_{j,k}=r(A_{j,k},X_{j,k},H_{j,k})\in[r\sub{min},r\sub{max}]
\end{equation}
with $0 < \Delta r=r\sub{max}-r\sub{min}<\infty$, and by the cost 
 \begin{equation}
	c_{j,k}=c(A_{j,k},X_{j,k},H_{j,k})\in[0,1].
\end{equation}
Both reward and cost  functions can exhibit arbitrary dependence on the selected action $A_{j,k}$, the current state $X_{j,k}$, and the within-block history $H_{j,k}$. However, they do not depend on actions or states preceding the current block, ensuring the mentioned separation between episodes.

For an action sequence $A_{1:T}$, define its cumulative reward in block $j$ as
\begin{equation}
	R_j(A_{1:T})=\sum_{k=1}^{n_j}r(A_{j,k},X_{j,k},H_{j,k}),
\end{equation}
and its average cost as
\begin{equation}
	C_j(A_{1:T})=\frac{1}{n_j}\sum_{k=1}^{n_j}c(A_{j,k},X_{j,k},H_{j,k}).
\end{equation}

\subsection{Design Criteria}
\label{sec:design_criteria}

Our objective is to design a policy $A_{j,k}=\Pi(H_{\tau(j,k)})$ that satisfies a per-block cost constraint, while providing competitive performance in terms of cumulative reward.
Specifically, for a fixed per-episode cost budget level $D\in[0,1]$, we seek a policy whose action sequence satisfies the per-block cost constraint
\begin{equation}
	\Pr\left[\max_{j\in[J]} C_j(A_{1:T})\le D\right]=1
	\label{eq:reliability-objective}
\end{equation}
for every state sequence, where the probability is taken with respect to the learner's own internal randomness.

At the same time, we wish the action sequence $A_{1:T}$ to be competitive against the best action sequence in a comparator class whose members satisfy the same per-block constraint~\eqref{eq:reliability-objective}. To construct such a comparator class, we assume that a zero-cost action is available at every round. Specifically, there exists an action $A^\star\in\mathcal A$ such that, for every state $x\in \mathcal X$ and every possible history $h$, zero cost is incurred, i.e., 
\begin{equation}
    \label{eq:safe-action}
    c(A^\star, x, h)=0.
\end{equation}

Fix a nominal action $a\in\mathcal A$, which can be chosen in hindsight based on the state sequence. We define the comparator class so that it deviates minimally from this constant action,  while still satisfying the constraint (\ref{eq:reliability-objective}). To elaborate, write as $\widetilde A_{1:T}$ the sequence of actions produced by this policy and as \begin{equation}
    \label{eq:cost-enforced-comparator-history}
    \widetilde H_{j,k} = \left(\widetilde A_{j,1:k-1}, X_{j,1:k-1} \right)
\end{equation} the corresponding history of actions and states. The action $ \widetilde A_{j,k}  $ follows the nominal selection $a$ provided that this choice is allowed by the constraint (\ref{eq:reliability-objective}), while switching to the zero-cost action otherwise. 

Formally, let the accumulated cost in episode $j$ at round $k$ be and let its remaining block budget be $\widetilde g_{j,k} = Dn_j-\widetilde z_{j,k}.$
The action executed at round $(j,k)$ is  given by
\begin{equation}
	\label{eq:cost-enforced-comparator}
    \widetilde A_{j,k} =
    \begin{cases}
        a, & \text{if }\widetilde g_{j,k}\ge 1,\\
        A^{\star}, & \text{if }\widetilde g_{j,k}<1,
    \end{cases}
\end{equation} ensuring the deterministic constraint \begin{equation}
    \max_{j\in[J]}
    C_j(\widetilde A_{1:T})
    \le D.
    \label{eq:comparator-cost-constraint}
\end{equation} 
As the nominal action $a$ ranges over set $\mathcal A$,  the comparator class is accordingly defined as 
\begin{multline}
    \label{eq:comparator-class}
    \mathcal C_T(D) =\bigl\{ \widetilde A_{1:T} :\exists a\in\mathcal A \text{ s.t. } \widetilde A_{j,k} \text{ as in \eqref{eq:cost-enforced-comparator} } \forall (j,k) \bigr\}.
\end{multline}

Any comparator sequence in this class achieves a reward in block $j$ given by
\begin{equation}
	\label{eq:cost-enforced-comparator-reward}
    R_j(\widetilde A_{1:T}) = \sum_{k=1}^{n_j} r(\widetilde A_{j,k}, X_{j,k}, \widetilde H_{j,k}).
\end{equation} Following~\cite{arora2012online}, we define the \emph{policy} regret against the comparator class~\eqref{eq:comparator-class} as
\begin{equation}
	\operatorname{Reg}^{\mathrm{pol}}_T=\max_{\widetilde A_{1:T} \in \mathcal C_T(D)}\sum_{j=1}^J R_j(\widetilde A_{1:T}) - \sum_{j=1}^J R_j(A_{1:T}).
	\label{eq:policy-regret}
\end{equation}
We wish to guarantee sublinear expected policy regret, i.e.,
\begin{equation}
	\mathbb{E}[\operatorname{Reg}^{\mathrm{pol}}_T]=o(T), \quad T\to\infty
	\label{eq:sublinear-policy-regret}
\end{equation}
for every state sequence, where the expectation is taken with respect to the learner's own internal randomness.

\subsection{Examples}\label{sec:examples}

We now describe two examples of the general framework under study, which will be further discussed in Section \ref{sec:applications}.

\paragraph*{1) Robust Model Predictive Control Under Safety Constraints}
Consider a dynamical system operated over a sequence of control episodes, with each block corresponding to one episode. The goal here is for the controller to maximize performance, while ensuring that the frequency at which the realized
safety risk exceeds the certificate produced by the policy remains below the budget level $D$ within every episode.

Formally, at round $(j,k)$, the system is affected by an exogenous disturbance $X_{j,k}\in\mathcal X$, which is \emph{a priori} unknown. The action $A_{j,k}$ specifies a candidate model predictive controller $m_{A_{j,k}}(H_{j,k})$ and an associated safety certificate $\nu_{A_{j,k}}(H_{j,k})$. Let $s_{j,k}(m_{A_{j,k}}(H_{j,k}),X_{j,k})$ denote the realized safety risk, with larger values corresponding to greater risk. After state $X_{j,k}$ is revealed, the incurred cost is
\begin{multline}
	c(A_{j,k},X_{j,k},H_{j,k})     \\                                                 
	 =\mathds{1}\{s_{j,k}(m_{A_{j,k}}(H_{j,k}),X_{j,k})>\nu_{A_{j,k}}(H_{j,k})\}.\label{eq:risk-outage}
\end{multline}
Thus, the cost is $c_{j,k}=1$ whenever the realized risk exceeds the associated safety certificate and $c_{j,k}=0$ otherwise.
The reward $r_{j,k}$ can be any function that characterizes the control performance. 

\paragraph*{2) Online Compression Under Distortion Guarantees}

Following \cite{ganesan2025online,zecchin2026prediction}, we consider a prediction-aided compression
setting, in which encoder and decoder share a common predictor. At the beginning of every block, the
encoder and decoder reset their common predictor to the same initial state. Within block~$j$, they
maintain a common state $S_{j,k}$ containing the previously reconstructed symbols.
Using this state,
they form a predictive distribution $\widehat p_{j,k}(\cdot\mid S_{j,k})$ for the next source symbol~$X_{j,k}$.

Before observing $X_{j,k}$, the encoder and decoder select an action $A_{j,k}$, which determines an
encoder--decoder pair. Let $\widehat X_{j,k}$ denote the resulting reconstruction. The
cost is measured by the outage distortion 
\begin{equation}
	c(A_{j,k},X_{j,k},H_{j,k})=\mathds{1}\{X_{j,k}\neq \widehat X_{j,k}\},
\end{equation}
while the reward can be defined as the negative of the transmission rate. Accordingly, the design objective is to minimize
the number of transmitted bits, while satisfying the distortion requirement separately within every block.

\section{Decision Making with Cost and Reward Guarantees}
\label{sec:blockwise-exp3}

In this section, we present  a policy $\Pi$ that satisfies the
per-block constraint~\eqref{eq:reliability-objective}, while also exhibiting sublinear expected policy regret as per requirement \eqref{eq:sublinear-policy-regret}. We refer to the approach as OCEA. In Section \ref{sec:applications}, we will then specialize this general policy to the examples described in the previous section.

At the beginning of each block $j$, the decision maker draws an action index $I_j \in [K]$ according to a probability distribution~$\pi_j$, and then plays that
action throughout the block, i.e., $A_{j,k}=a_{I_j}$, provided that the cost constraint~\eqref{eq:reliability-objective} is not violated. In so doing, the decision maker adopts the same general constraint-enforcing approach as the comparator policy class introduced in Section~\ref{sec:design_criteria} with the key caveat that the decision maker does not have access to future states, operating in a causal fashion.

Specifically, the probability distribution $\pi_j$ is updated online at the end of the block via the EXP3 algorithm~\cite{auer2002nonstochastic}. Accordingly, given a weight vector $\boldsymbol w_j = \bigl[w_j(1),\ldots,w_j(K)\bigr]$, with $w_1(i) = 1$ for $i\in [K]$, the probability is set as 
\begin{equation}
	\label{eq:exp3-arm-distribution}
	\pi_j(i) = (1-\gamma)\frac{w_j(i)} {\sum_{r=1}^{K}w_j(r)} + \frac{\gamma}{K},
\end{equation}
where $\gamma \in (0, 1]$ is an exploration parameter. 
\enlargethispage{-0.04in}

Let the learner's accumulated cost before round~$k$ of block~$j$ be $z_{j,k} = \sum_{s=1}^{k-1} c(A_{j,s},X_{j,s},H_{j,s})$, with $z_{j,1}=0$ and the remaining block budget be $g_{j,k}=Dn_j-z_{j,k}$. The learner executes the action
\begin{equation}
    \label{eq:constraint-enforced-action}
    A_{j,k} =
    \begin{cases}
        a_{I_j}, & \text{if }g_{j,k}\ge 1,\\
        A^\star, & \text{if }g_{j,k}<1,
    \end{cases}
\end{equation}
applying the same rule used to construct the comparator sequences in~\eqref{eq:cost-enforced-comparator}, with $a_{I_j}$ as its nominal action.

For each $i\in[K]$, let $R_j(i)$ denote the counterfactual reward in block $j$ obtained by playing~\eqref{eq:constraint-enforced-action} with $a_i$ as the nominal action and evaluating the reward of the resulting actions along their own counterfactual history. 
Only the reward
\begin{equation}
    \label{eq:observed-block-reward}
    R_j(I_j)=\sum_{k=1}^{n_j}r(A_{j,k},X_{j,k},H_{j,k})
\end{equation}
is observed by the learner.

Let $n_{\max}=\max_{j\in[J]}n_j<\infty$. Since the per-round reward lies in the interval $[r\sub{min},r\sub{max}]$, the per-block reward satisfies
\begin{equation}
	n_jr\sub{min} \le R_j(i) \le n_jr\sub{max}.
\end{equation}
This allows us to define the normalized per-block reward
\begin{equation}
	\label{eq:normalized-block-reward}
	\bar R_j(i) = \frac{R_j(i)-n_jr\sub{min}} {\Delta r ~ n_{\max}} \in[0,1],
\end{equation}
based on which the estimated weighted rewards for all actions $a_i$ with $i\in [K]$ are  obtained as 
\begin{equation}
	\label{eq:importance-weighted-block-reward}
	\widehat R_j(i)
	=
	\frac{\bar R_j(I_j)\mathds{1}\{I_j=i\}}
	{\pi_j(i)}.\end{equation}
These estimates are used to update the weights to increase the weights of actions with larger estimated rewards:
\begin{equation}
	\label{eq:exp3-weight-update}
	w_{j+1}(i) = w_j(i) \exp\Bigl( \frac{\gamma}{K} \widehat R_j(i) \Bigr), \qquad i\in[K].
\end{equation}
We have the following result.
\begin{theorem}
	\label{thm:blockwise-exp3}
	Fix an arbitrary exogenous state sequence $X_{1:T}$ and suppose that $n_{\max}, K, \Delta r$ are fixed and independent of~$T$.
	Under the zero-cost assumption \eqref{eq:safe-action}, the actions generated by the OCEA policy \eqref{eq:constraint-enforced-action} satisfy the constraint \eqref{eq:reliability-objective} for every realization of the learner's randomization. Moreover, let
	\begin{equation}\label{eq:gamma}
		\gamma  = \min\Biggl\{ 1, \sqrt{\frac{K\ln K}{(e-1)J}} \Biggr\}.
	\end{equation}
	The expected policy regret satisfies
	\begin{multline}
		\max_{i\in[K]} \sum_{j=1}^{J}R_j(i) - \mathbb E \Biggl[ \sum_{j=1}^{J}R_j(I_j) \Biggr] \\
			\le 2.63 \Delta r n_{\max}\sqrt{KT\ln K}, \label{eq:exp3-policy-regret-bound}
	\end{multline}
	thus meeting the requirement \eqref{eq:sublinear-policy-regret}.
\end{theorem}
\begin{IEEEproof}
	See Appendix~\ref{app:proof-blockwise-exp3}.
\end{IEEEproof}
It is important to emphasize that resetting the memory at the beginning of each block is crucial for obtaining the sublinear-in-$T$ expected policy-regret bound~\eqref{eq:exp3-policy-regret-bound}.
Indeed, as shown in~\cite{arora2012online}, when the memory is unbounded, it is possible to adversarially choose reward functions and state sequence that result in linear policy regret.

\section{Applications }
\label{sec:applications}

We next illustrate how to apply the OCEA policy described in Section~\ref{sec:blockwise-exp3} to
the two examples introduced in Section~\ref{sec:examples}.
In both cases, we will consider scenarios in which the decision maker has access to a state predictor, which can be used to inform its action policy.

\subsection{Robust Model Predictive Control under Safety Constraints}
\label{sec:risk-averse}

Specializing the example described in Section~\ref{sec:examples}, we consider a variation of the risk-averse decision-making problem introduced in~\cite{kiyani2025decision} (see also~\cite{farzaneh2026}) in which the action determines the level of risk aversion of the controller through the selection of a predictive controller. 

To elaborate, assume that the decision maker uses a state predictor to infer the state in round $(j,k)$ given the available in-block history $H_{j,k}$.
Specifically, in round $(j,k)$ the predictor $\widehat p_{j,k}$ returns the top-probability prediction set \begin{equation}\mathcal{X}_{j,k}=\{ x\in \mathcal X: \widehat p_{j,k}(x|H_{j,k})\geq A_{j,k}\}\end{equation} for a suitably chosen threshold  $A_{j,k}$, which acts as the action. 
Note that this construction is central to online conformal prediction methods~\cite{gibbs2021adaptive,angelopoulos2024online,angelopoulos2024conformal,zecchin2024localized}. 
Then, the controller $m_{A_{j,k}} \in \mathcal{M}$ is selected so as to minimize the worst-case safety-risk function $s_{j,k}(m,x)$ among all states in the predicted set $\mathcal X_{j,k}$, i.e., 
\begin{equation}
m_{A_{j,k}} = \argmin_{m \in \mathcal{M}} \max_{x\in \mathcal X_{j,k}} s_{j,k}(m,x).
\end{equation}
With this choice, the safety certificate is set to the corresponding worst-case safety-risk level
\begin{equation}
\nu_{A_{j,k}} = \min_{m\in\mathcal{M}} \max_{x\in \mathcal X_{j,k}} s_{j,k}(m,x).
\end{equation}

Following this approach, the cost~\eqref{eq:risk-outage} equals $c_{j,k}=0$ as long as the true state is in the predicted set, i.e., $X_{j,k} \in \mathcal X_{j,k}$. Furthermore, the zero-cost action is $A^\star_{j,k}=0$, yielding the full set $\mathcal X_{j,k} = \mathcal X$. Therefore, OCEA, selecting the predictive threshold  as described in the previous section, achieves per-episode cost constraints and sublinear expected policy regret.

\subsection{Online  Compression Under Distortion Guarantees}
\label{sec:online-compression}

For the online compression setting described in Section \ref{sec:examples}, the top-probability prediction set at round $(j,k)$ is given by $\mathcal X_{j,k} = \{ x\in\mathcal X: \widehat p_{j,k}(x\mid S_{j,k})\ge A_{j,k} \}$.
Following \cite{zecchin2026prediction}, a lossless prefix-free entropy code is constructed from the augmented truncated distribution
\begin{equation}
	\label{eq:ocsc-encoding-distribution}
	q_{j,k}(x\mid S_{j,k}) =
	\begin{cases}
		\dfrac{(1-D) \widehat p_{j,k}(x\mid S_{j,k})} {\scriptstyle\sum_{x'\in\mathcal X_{j,k}} \widehat p_{j,k}(x'\mid S_{j,k})}, & x\in\mathcal X_{j,k},\\
		D, & x=x_{\mathrm o},
	\end{cases}
\end{equation}
where $x_{\mathrm o}\notin \mathcal X$ is an outage symbol. If $X_{j,k}\in\mathcal X_{j,k}$, the encoder transmits the codeword corresponding to $X_{j,k}$ and the decoder reconstructs it exactly. Otherwise, the encoder transmits the codeword corresponding to $x_{\mathrm o}$, and the decoder outputs $\widehat X_{j,k} = \argmax_{x\notin\mathcal X_{j,k}} \widehat p_{j,k}(x\mid S_{j,k})$ where ties are broken according to a fixed criterion. 

The reward is the negative self-information of the encoded symbol, $r(A_{j,k},X_{j,k},H_{j,k}) = \log_2 q_{j,k}(\widetilde X_{j,k}\mid S_{j,k})$, where $\widetilde X_{j,k}=X_{j,k}$ if $X_{j,k}\in\mathcal X_{j,k}$ and $\widetilde X_{j,k}=x_{\mathrm o}$ otherwise. 
With these definitions, a direct application of OCEA meets the per-block erasure distortion constraints, while also being competitive in terms of achieved compression rate as measured by the policy regret.
In contrast, the online conformal sparse compression (OCSC) policy introduced in~\cite{zecchin2026prediction} provides only long-term outage-distortion guarantees and no guarantees on the compression rate.

\section{Experiments}
\label{sec:experiments}

We evaluate OCEA in the online compression setting described in Section~\ref{sec:online-compression}, and compare it with OCSC~\cite{zecchin2026prediction} on a synthetic compression problem. We refer to~\cite{zecchin2026prediction} for the complete description of OCSC and specify below the state sequence (source), predictor, and algorithm parameters used in the experiment.

We choose $\mathcal X=\{a,b,c\}$, assume that $X_t=a$ for every~$t$, and set $D=1/4$. We divide the time horizon into blocks of four symbols. At the beginning of every block $j$, the predictor is reset to the public state $S_{j,1}=a$. Within the block, the predictor depends on the most recently reconstructed symbol (i.e., $S_{j,k} = \widehat X_{j,k-1}$) according to
\begin{equation}
	\widehat p_{j,k}(\cdot\mid S_{j,k}) =
	\begin{cases}
		P_a, & S_{j,k}\in\{a,c\},\\
		P_b, & S_{j,k}=b,
	\end{cases}
\end{equation}
where $P_a= ( \tfrac12-\tfrac{\eta}{4}, \tfrac12-\tfrac{\eta}{8}, \tfrac{3\eta}{8} )$, $P_b = ( \tfrac12-\tfrac{3\eta}{4}, \tfrac12-\tfrac{3\eta}{4}, \tfrac{3\eta}{2} )$, with coordinates ordered as $(a,b,c)$. Here, $\eta$ is the OCSC step size~\cite{zecchin2026prediction}. We set $T\in\{2^6,2^7,\ldots,2^{18}\}$ and $\eta=T^{-1/2}$. For OCEA, we let $K=10$ and consider the candidate thresholds $\big\{ a_i = \frac{i-1}{K-1} \big\}_{i=1}^K$. The EXP3 exploration parameter~$\gamma$ is set according to \eqref{eq:gamma}. 
For reward normalization, we use the interval $[r_{\min},r_{\max}]=[\log_2(9\eta/32),\log_2(3/4)]$, which contains every per-symbol reward generated by the predictor.
The comparator sequence is generated according to~\eqref{eq:cost-enforced-comparator} using the nominal threshold $A^\dagger=1/9$, which maximizes the cumulative reward over the comparator set.

In Fig.~\ref{fig:synthetic-regret}, we report the normalized policy regret $\operatorname{Reg}^{\mathrm{pol}}_T/T$ for both OCSC (with threshold initialized to $1/2$) and OCEA. The OCEA results are averaged over $128$ independent realizations of the EXP3 randomization, with error bars denoting one standard deviation.

As illustrated in the figure, the normalized policy regret of OCSC remains bounded away from zero, whereas that of OCEA decreases as the time horizon grows.
The policy regret behavior of OCSC can actually be characterized analytically. Specifically, when OCSC's initial threshold is set to $1/2$ and when $\eta\in(0,1/6]$, one can show that, for every time horizon~$T$ divisible by four, its policy regret relative to the fixed action $A^\dagger$ satisfies $\operatorname{Reg}^{\mathrm{pol}}_T(\mathrm{OCSC}) \ge \frac{T}{4}\log_2\frac{968}{675}$.
Thus, its normalized regret is lower-bounded by the positive constant $\frac14\log_2(968/675)$, shown by the dashed line in Fig.~\ref{fig:synthetic-regret}.
By contrast, the decreasing normalized regret of OCEA agrees with the sublinear expected policy-regret guarantee in Theorem~\ref{thm:blockwise-exp3}.

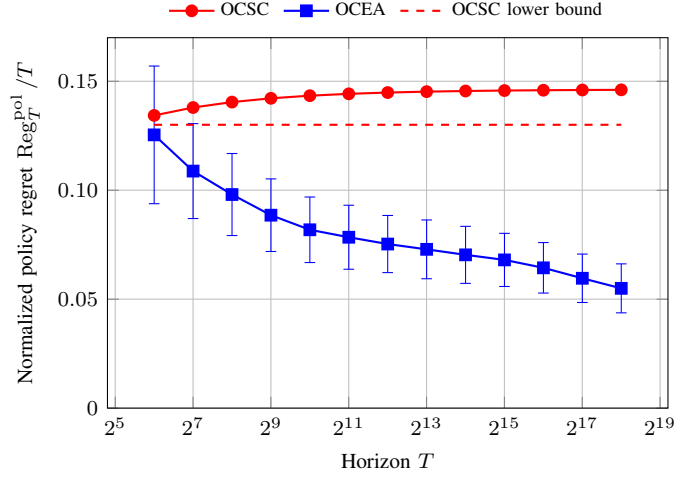
\begin{figure}[t]
    \centering
	\begin{tikzpicture}
    \begin{axis}[
        font=\footnotesize, 
        width=\columnwidth,
        height=0.72\columnwidth,
        xmode=log,
        log basis x=2,
        xlabel={Horizon $T$},
        ylabel={Normalized policy regret $\operatorname{Reg}^{\mathrm{pol}}_T/T$},
        grid=major,
        ymin=0,
        ymax=0.17,
        ytick={0,0.05,0.10,0.15},
        yticklabels={$0$,$0.05$,$0.10$,$0.15$},
        scaled y ticks=false,
        legend style={
            at={(0.5,1.02)},
            anchor=south,
            legend columns=3,
            font=\scriptsize,
            draw=none,
            /tikz/every even column/.append style={column sep=3pt}
        },
        mark size=2pt,
    ]

    \addplot[
        red,
        thick,
        mark=*,
    ] table[
        col sep=comma,
        x=horizon,
        y=ocsc_regret_per_symbol
    ] {figs/counterexample_results.csv};
    \addlegendentry{OCSC}

    \addplot[
        blue,
        thick,
        mark=square*,
        error bars/.cd,
        y dir=both,
        y explicit,
    ] table[
        col sep=comma,
        x=horizon,
        y=ocea_regret_per_symbol,
        y error=ocea_regret_standard_deviation_per_symbol
    ] {figs/counterexample_results.csv};
    \addlegendentry{OCEA}

    \addplot[
        red,
        dashed,
        thick,
        mark=none,
    ] table[
        col sep=comma,
        x=horizon,
        y=ocsc_lower_bound
    ] {figs/counterexample_results.csv};
    \addlegendentry{OCSC lower bound}

    \end{axis}
\end{tikzpicture}
    \caption{Normalized policy regret as a function of the time horizon $T$ for OCSC and OCEA. The dashed line indicates the OCSC regret lower bound $\frac14\log_2(968/675)$. The OCEA curve corresponds to the mean over 128 independent runs, with error bars showing one standard deviation.}
    \label{fig:synthetic-regret}
\end{figure}

\section{Conclusion}

In this paper, we introduced a sequential decision framework for maximizing cumulative reward subject to deterministic per-episode cost constraints. We proposed OCEA, an EXP3-based policy that selects an action for each episode and switches to a zero-cost action when the remaining budget cannot accommodate the maximum cost of another round. OCEA guarantees that the per-episode cost constraint is satisfied for every state sequence and every realization of its randomization, while achieving sublinear expected policy regret. We specialized the framework to  risk-averse decision making and online compression. Numerical results for online compression demonstrate that the normalized regret of OCEA decreases with the time horizon, whereas that of a state-of-the-art method can remain bounded away from zero.

\appendices

\section{Proof of Theorem \ref{thm:blockwise-exp3}}
\label{app:proof-blockwise-exp3}

The cost constraint follows directly from the enforcement rule~\eqref{eq:constraint-enforced-action}. It remains to prove the policy-regret bound.
Because the state sequence is exogenous and the history is reset at the beginning of every block, the normalized rewards $\{\bar R_j(i)\}_{j\in[J],i\in[K]}$ form a fixed table in $[0,1]$.
The sampling rule, reward estimator, and weight update are therefore those of EXP3. Hence, \cite[Cor.~3.2]{auer2002nonstochastic} gives
\begin{equation}
	\label{eq:normalized_reward_regret}
	\max_{i\in[K]}\sum_{j=1}^J\bar R_j(i) -\mathbb E\Biggl[\sum_{j=1}^J\bar R_j(I_j)\Biggr] \le 2.63\sqrt{KJ\ln K}.
\end{equation}
It follows from~\eqref{eq:normalized-block-reward} that to express the regret in terms of the $\{R_j(i)\}$ we need to multiply the right-hand side of~\eqref{eq:normalized_reward_regret} by $\Delta r n_{\max}$.
Since $J\le T$, this proves \eqref{eq:exp3-policy-regret-bound}.
Finally, because $K$, $\Delta r$, and $n_{\max}$ are finite and independent of $T$, the resulting bound is sublinear in $T$.

\section*{AI Statement}

The authors used ChatGPT Sol 5.6 (OpenAI) to  develop the numerical example in Section~\ref{sec:experiments} and to establish that, for this example, OCSC exhibits linear policy regret.

\bibliographystyle{IEEEtran}
\bibliography{references}

@article{zecchin2026prediction,
    author = {Zecchin, Matteo and Ganesan, Unnikrishnan Kunnath and Durisi, Giuseppe and Popovski, Petar and Simeone, Osvaldo},
    title = {Prediction-Powered Communication With Distortion Guarantees}, 
    journal = {IEEE J. Sel. Areas Inf. Theory}, 
    month = {Feb.},
    year = {2026},
    volume = {7},
    number = {},
    pages = {33-45},
}

@inproceedings{gibbs2021adaptive,
    author = {Gibbs, Isaac and Cand{\`e}s, Emmanuel J.},
    title = {Adaptive conformal inference under distribution shift},
    booktitle = {Proc. Conf. Neural Inf. Process. Syst. (NeurIPS)},
    address = {Virtual Conference},
    month = {Dec.},
    year = {2021},
}

@inproceedings{arora2012online,
    author = {Arora, Raman and Dekel, Ofer and Tewari, Ambuj},
    title = {Online bandit learning against an adaptive adversary: from regret to policy regret},
    booktitle = {Proc. Int. Conf. Mach. Learning (ICML)},
    address = {Edinburgh, Scotland},
    month = {June},
    year = {2012},
}

@article{farzaneh2026,
    author  = {Farzaneh, Amirmohammad and Simeone, Osvaldo},
    title   = {Risk-Averse Decision Making with Multi-Level Reliability Guarantees},
    journal = {arXiv:2609.11524},
    year    = {2026},
}

@inproceedings{kiyani2025decision,
    author = {Kiyani, Shayan and Pappas, George J. and Roth, Aaron and Hassani, Hamed},
    title = {Decision Theoretic Foundations for Conformal Prediction: Optimal Uncertainty Quantification for Risk-Averse Agents},
    booktitle = {Proc. Int. Conf. Mach. Learning (ICML)},
    address = {Vancouver, Canada},
    month = {July},
    year = {2025},
}

@inproceedings{ganesan2025online,
    author = {Ganesan, Unnikrishnan Kunnath and Durisi, Giuseppe and Zecchin, Matteo and Popovski, Petar and Simeone, Osvaldo},
    title = {Online conformal compression for zero-delay communication with distortion guarantees},
    booktitle = {Proc. IEEE Int. Symp. Inf. Theory (ISIT)},
    address = {Ann Arbor, Michigan, USA},
    month = {June},
    year = {2025},
}

@article{auer2002nonstochastic,
    author = {Auer, Peter and Cesa-Bianchi, Nicol{\`o} and Freund, Yoav and Schapire, Robert E.},
    title = {The Nonstochastic Multiarmed Bandit Problem},
    journal = {SIAM J. Comput.},
    month = {},
    year = {2002},
    volume = {32},
    number = {1},
    pages = {48-77},
}

@inproceedings{angelopoulos2024online,
    author    = {Angelopoulos, Anastasios N. and Barber, Rina Foygel and Bates, Stephen},
    title     = {Online Conformal Prediction with Decaying Step Sizes},
    booktitle = {Proc. Int. Conf. Mach. Learn. (ICML)},
    address   = {Vienna, Austria},
    month     = {July},
    year      = {2024},
}

@inproceedings{angelopoulos2024conformal,
    author    = {Angelopoulos, Anastasios N. and Bates, Stephen and Fisch, Adam and Lei, Lihua and Schuster, Tal},
    title     = {Conformal Risk Control},
    booktitle = {Proc. Int. Conf. Learn. Represent. (ICLR)},
    address   = {Vienna, Austria},
    month     = {May},
    year      = {2024},
}

@inproceedings{zecchin2024localized,
    author    = {Zecchin, Matteo and Simeone, Osvaldo},
    title     = {Localized Adaptive Risk Control},
    booktitle = {Proc. Conf. Neural Inf. Process. Syst. (NeurIPS)},
    address   = {Vancouver, Canada},
    month     = {Dec.},
    year      = {2024}
}

@inproceedings{lekeufack2024conformal,
    author    = {Lekeufack, Jordan and Angelopoulos, Anastasios N. and Bajcsy, Andrea and Jordan, Michael I. and Malik, Jitendra},
    title     = {Conformal Decision Theory: Safe Autonomous Decisions from Imperfect Predictions},
    booktitle = {Proc. IEEE Int. Conf. Robot. Autom. (ICRA)},
    address   = {Yokohama, Japan},
    month     = {May},
    year      = {2024},
}

@article{lindemann2023safe,
    author  = {Lindemann, Lars and Cleaveland, Matthew and Shim, Gihyun and Pappas, George J.},
    title   = {Safe Planning in Dynamic Environments Using Conformal Prediction},
    journal = {IEEE Robot. Autom. Lett.},
    month   = {Aug.},
    year    = {2023},
    volume  = {8},
    number  = {8},
    pages   = {5116--5123},
}

@article{lindemann2025formal,
    author  = {Lindemann, Lars and Zhao, Yiqi and Yu, Xinyi and Pappas, George J. and Deshmukh, Jyotirmoy V.},
    title   = {Formal Verification and Control with Conformal Prediction: Practical Safety Guarantees for Autonomous Systems},
    journal = {IEEE Control Syst. Mag.},
    year    = {2025},
    month   = {Dec.},
    volume  = {45},
    number  = {6},
    pages   = {72--122},
}

@article{zecchin2024forking,
    author  = {Zecchin, Matteo and Park, Sangwoo and Simeone, Osvaldo},
    title   = {Forking Uncertainties: Reliable Prediction and Model Predictive Control With Sequence Models via Conformal Risk Control},
    journal = {IEEE J. Sel. Areas Inf. Theory},
    month   = {Feb.},
    year    = {2024},
    volume  = {5},
    number  = {},
    pages   = {44--61},
}

@inproceedings{cairoli2023conformal,
    author    = {Cairoli, Francesca and Paoletti, Nicola and Bortolussi, Luca},
    title     = {Conformal Quantitative Predictive Monitoring of {STL} Requirements for Stochastic Processes},
    booktitle = {Proc. ACM Int. Conf. Hybrid Syst.: Comput. Control (HSCC)},
    address   = {San Antonio, TX, USA},
    month     = {May},
    year      = {2023},
}

@inproceedings{zhu2026conformal,
    author    = {Zhu, Zihan and Kiyani, Shayan and Pappas, George J. and Hassani, Hamed},
    title     = {Conformal Risk-Averse Decision Making with Action Conditional Guarantee},
    booktitle = {Proc. Int. Conf. Mach. Learning (ICML)},
    address   = {Seoul, South Korea},
    month     = {July},
    year      = {2026}
}

@article{simeone2026conformal,
    author  = {Simeone, Osvaldo and Park, Sangwoo and Zecchin, Matteo},
    title   = {Conformal Calibration: Ensuring the Reliability of Black-Box {AI} in Wireless Systems},
    journal = {IEEE Commun. Mag.},
    year    = {2026},
    month   = {May},
    volume  = {64},
    number  = {8},
    pages   = {179--187},
}

@article{mahdavi2012trading,
    author  = {Mahdavi, Mehrdad and Jin, Rong and Yang, Tianbao},
    title   = {Trading Regret for Efficiency: Online Convex Optimization with Long Term Constraints},
    journal = {J. Mach. Learn. Res.},
    month   = {Sep.},
    year    = {2012},
    volume  = {13},
    number  = {1},
    pages   = {2503--2528},
}

@inproceedings{yu2017online,
    author    = {Yu, Hao and Neely, Michael J. and Wei, Xiaohan},
    title     = {Online Convex Optimization with Stochastic Constraints},
    booktitle = {Proc. Conf. Neural Inf. Process. Syst. (NeurIPS)},
    address   = {Long Beach, CA, USA},
    month     = {Dec.},
    year      = {2017}
}
\end{document}